\documentclass{article}
\usepackage{enumitem}
\usepackage{algorithm}
\usepackage{algpseudocode}
\usepackage{graphicx}
\usepackage[round]{natbib}
\usepackage{subcaption}
\usepackage{amsmath}
\usepackage{amssymb}
\usepackage{amsthm}
\usepackage{hyperref}

\newtheorem{pro}{Proposition}

\algnewcommand{\LComment}[1]{\State $\triangleright$ #1}
\newcommand{\mathsc}[1]{\text{\normalfont\textsc{#1}}}
\newcommand{\dom}{\operatorname{dom}}

\title{MDP Planning as Policy Inference}

\author{David Tolpin,\\Offtopia\\ david.tolpin@gmail.com}

\begin{document}

\maketitle


\begin{abstract}
We formulate episodic Markov decision process (MDP) planning as
Bayesian inference over \emph{policies}. The primary contribution is
conceptual: the policy itself is treated as the latent variable, and
expected return defines an unnormalized posterior density over
policies. This preserves the standard expected-return objective, in
contrast to trajectory-centric planning-as-inference formulations that
introduce auxiliary optimality variables and to entropy-regularized
policy optimization methods that solve a different objective.

In the exact formulation, the posterior over deterministic policies
induces what we define here as an \emph{optimal stochastic policy under
preference uncertainty}, namely the stochastic policy induced by that
posterior. For discrete MDPs
with stochastic transitions, we study variational sequential Monte Carlo
(VSMC) as one approximate inference method for this posterior,
introducing policy consistency under state revisitation and coupled
transition randomness across particles.

Experiments on grid worlds, Blackjack, Triangle Tireworld, and Academic
Advising examine the consequences of inference over policies and compare
its induced behavior with entropy-regularized policy optimization. The
results support the view that MDP planning can be naturally cast as
Bayesian inference over policies.
\end{abstract}

\section{Introduction}

We cast episodic Markov decision process (MDP) planning as Bayesian
inference over \emph{policies}. The central claim of the paper is that
the natural object of inference in planning is the policy itself: not a
trajectory augmented with auxiliary optimality variables, and not a
single stochastic policy trained under entropy regularization, but a
posterior distribution over policies whose density is determined by
expected return. Treating the policy as a latent variable has two
benefits: it enables the use of general-purpose inference algorithms
for planning, and it makes uncertainty about optimal behavior explicit
as posterior dispersion rather than an artifact of approximation or
heuristic regularization.

Casting planning and control as probabilistic inference is useful because
it provides a principled modeling language for uncertainty and makes
general-purpose approximate inference tools available for decision
making. This perspective has a long history in probabilistic planning,
control-as-inference, maximum-entropy reinforcement learning, and active
inference \citep{DH97, TS06, Z10, L18}. Prior formulations typically
modify the classical planning objective to fit a standard
latent--observation template, for example by introducing
entropy-regularized or evidence-based surrogate criteria. In these
settings, stochasticity is often a modeling preference or an exploration
device, and uncertainty over the solution of the original
expected-return problem is not directly interpretable from the inferred
policy.

This distinction in the object of inference matters. In trajectory-centric
planning-as-inference methods, randomness is attached to latent
trajectories and auxiliary optimality variables are introduced to
recover a control objective. In entropy-regularized reinforcement
learning, stochasticity is built directly into the optimized policy and
reflects a modified objective. Here, by contrast, stochasticity arises
from posterior uncertainty over deterministic policies, while the
underlying objective remains the standard expected-return criterion.

In this work, we propose a Bayesian formulation that preserves the
standard MDP objective.  We define an unnormalized probability of
optimality for each policy that is monotone in its expected return,
yielding a posterior distribution whose modes coincide with
return-maximizing policies, while posterior dispersion represents
uncertainty over optimal behavior.  Acting is performed by sampling from
the posterior predictive distribution induced by this policy posterior.
This yields what we define in
Section~\ref{sec:prob-model-induced-policy} as the model-induced
optimal stochastic policy under preference uncertainty, through a
Thompson-sampling interpretation rather than through
entropy regularization.

Once planning is formulated this way, approximate Bayesian inference is
needed to work with the resulting posterior. In this paper we use
variational sequential Monte Carlo (VSMC) \citep{NLR+18,MLT+17} as one
approximate inference tool for discrete MDPs with stochastic
transitions. The algorithmic development is therefore in service of the
policy-inference formulation, rather than the main contribution of the
paper. Concretely, our adaptation enforces policy consistency across
revisited states and couples transition randomness across particles
within a sweep so that particle weights reflect policy differences
rather than independent realizations of simulator noise. We further
interpret the reward scale as controlling uncertainty over agent
preferences: larger rewards induce posterior concentration and
near-deterministic behavior, while smaller rewards yield a diffuse
posterior and a stochastic control policy that reflects preference
uncertainty through the inferred variational posterior. This
interpretation treats reward magnitudes as meaningful, not merely as an
ordering over policies: a small return difference expresses weak
preference between policies, whereas a large return difference expresses
strong preference and therefore induces stronger posterior concentration.

\paragraph{What this paper is and is not.}
This paper is primarily about the \emph{object of inference}. Our claim is
that MDP planning can be cast naturally as Bayesian inference over
\emph{policies} while preserving the standard expected-return semantics.
Accordingly, VSMC is used here as one approximate inference method for the
resulting posterior; it is not the main conceptual contribution, nor do we
present the method as a new entropy-regularized or ``soft'' planning
objective. In contrast to trajectory-centric planning-as-inference methods,
we do not introduce optimality variables or fictitious observations, and in
contrast to entropy-regularized RL we do not optimize a single stochastic
policy under an added entropy term.

Our contributions are:
\begin{itemize}
\item A formulation of episodic MDP planning as Bayesian inference over
\emph{policies}, in which the policy itself is the latent variable and
expected return defines an unnormalized posterior density. In the exact
formulation, this posterior induces what we define here as an
\textbf{optimal stochastic policy under preference uncertainty}, namely,
the stochastic policy obtained by marginalizing posterior uncertainty
over deterministic policies, without introducing trajectory-level
optimality variables or entropy regularization.
\item An adaptation of VSMC for inference over deterministic policies in
discrete MDPs with stochastic transitions, serving as an approximate
inference algorithm for this policy posterior and including policy
consistency under revisitation and coupled transition randomness across
particles.
\item An empirical evaluation of the \emph{induced stochastic control
policy} obtained by posterior predictive (Thompson-style) action
sampling, and a comparison to discrete Soft Actor-Critic across
diverse discrete benchmarks.
\end{itemize}

\section{Background}
\label{sec:background}

\subsection{Markov Decision Process}

We consider episodic MDPs with state space $S$, action space
$A$, stochastic transition kernel $\mathcal{T}(s' \mid s,a)$, reward
function $R(s,a,s')$, initial state $s_1$, and a set of absorbing
goal states $G \subseteq S$. A (Markov) policy $\pi$ maps states
to actions; we write either $\pi(s)\in A$ for a deterministic
policy or $p_\pi(a\mid s)$ for a stochastic policy. Episodes end
upon first reaching a state in $G$. This absorbing-goal formulation is
standard for episodic, goal-directed tasks, where interaction ends upon
reaching a terminal state and policies are evaluated by expected return
up to termination \citep{SB98,P94}. For planning and inference,
we use a finite rollout horizon $H$ as an algorithm parameter and
evaluate policies on truncated trajectories
$\tau_\pi=(s_1,a_1,s_2,\ldots,a_H,s_{H+1})$ generated by iterating,
for $t=1,\ldots,H$,
\begin{equation}
s_{t+1} \gets \mathsc{Step}(s_t,a_t), \qquad r_t \gets \mathsc{Reward}(s_t,a_t,s_{t+1}),
\label{eqn:mdp-simulator}
\end{equation}
where $\mathsc{Step}$ samples $s_{t+1}\sim \mathcal{T}(\cdot\mid s_t,a_t)$
and $\mathsc{Reward}$ computes $R(s_t,a_t,s_{t+1})$.
If $s_t \in G$, the process remains in that absorbing state and
accrues no further reward, so the truncated rollout continues only
for notational convenience. If no goal state is reached by time
$H$, the rollout is simply truncated. Throughout, inference
algorithms access the MDP only through this simulator interface.

\subsection{Variational Sequential Monte Carlo}
\label{sec:bg-vsmc}

Variational sequential Monte Carlo (VSMC)~\citep{NLR+18} is an
algorithm for variational inference in models amenable to sequential
Monte Carlo (SMC)~\citep{GSS93}. In standard SMC, the posterior
distribution of latent trajectories is approximated sequentially by
$N$ weighted particles. At each step, particles representing partial
latent trajectories are extended by sampling from a proposal kernel,
assigned incremental importance weights, and optionally resampled
according to those weights. An SMC sweep therefore produces two objects:
a weighted particle approximation to the posterior, and an estimate of
the marginal likelihood, or evidence. VSMC uses the latter object as a
variational training signal: it treats the parametrized SMC proposal as
a variational family and optimizes the expected SMC log-evidence
estimate, with the expectation taken over proposal sampling and
resampling in the SMC sweep.

Let $x_t\in\mathcal X_t$ be latent variables and
$y_t\in\mathcal Y_t$ be observations, with
$(\mathcal X_t,\mathcal Y_t)$ measurable spaces. Consider the
state-space model
\begin{equation}
p(x_{1:H},y_{1:H})
=
p(x_1)p(y_1\mid x_1)
\prod_{t=1}^{H-1}
p(x_{t+1}\mid x_t)p(y_{t+1}\mid x_{t+1})
\end{equation}
where $H$ is the horizon. The evidence is
\begin{equation}
Z = p(y_{1:H}).
\end{equation}
Here $q_\lambda$ denotes a parameterized family of proposal kernels on
the latent spaces, from which particles can be sampled sequentially and
whose densities or masses can be evaluated for importance weighting.
Given a proposal $q_\lambda$, SMC produces particles and weights
${x_{1:t}^{i},w_t^{i}}\big\rvert_{i=1}^{N}$ at each step $t$, and estimates
the evidence by
\begin{equation}
\hat Z_N
=
\prod_{t=1}^{H}
\frac{1}{N}\sum_{i=1}^{N} w_t^{(i)} .
\label{eqn:smc-Z-estimate}
\end{equation}
The estimator $\hat Z_N$ is unbiased for $Z$, while $\log \hat Z_N$ is a
biased estimate of $\log Z$ by Jensen's inequality. VSMC nevertheless
uses the log estimate as a variational objective,
\begin{equation}
\mathcal L_{\mathrm{VSMC}}(\lambda)
=
\mathbb E_{\mathrm{SMC}(q_\lambda)}
\left[
\log \hat Z_N
\right],
\label{eqn:vsmc-objective}
\end{equation}
where the expectation is over proposal sampling and resampling in the
SMC sweep. The objective is a surrogate evidence lower bound for the
marginal likelihood~\citep{NLR+18}.

The proposal learned by VSMC should be understood as an approximation to
a posterior conditional over the next latent variable, not merely as a
locally adapted proposal depending only on the next observation. This is
not a new result of the present paper; it is a reformulation of the
variational interpretation of SMC developed by \citet{NLR+18}. We state
it explicitly here to fix notation and to make clear which proposal
distribution is being approximated before specializing the construction
to policy inference.

\begin{pro}[Posterior-conditional target of the VSMC proposal]
\label{pro:vsmc-target-proposal}
Let $p(x_{1:H}\mid y_{1:H})$ be the posterior over a latent trajectory.
For fixed observations $y_{1:H}$, consider an autoregressive proposal
family
\begin{equation}
q_\lambda(x_{1:H}\mid y_{1:H})
=
q_\lambda(x_1\mid y_{1:H})
\prod_{t=1}^{H-1}
q_\lambda(x_{t+1}\mid x_{1:t},y_{1:H}).
\end{equation}
The maximizer of the variational evidence bound
\begin{equation}
\mathcal L(\lambda)
=
\mathbb E_{q_\lambda(x_{1:H}\mid y_{1:H})}
\left[
\log p(x_{1:H},y_{1:H})
-
\log q_\lambda(x_{1:H}\mid y_{1:H})
\right]
\end{equation}
satisfies, whenever the proposal family is expressive enough,
\begin{equation}
q_\lambda^\star(x_1\mid y_{1:H})
=
p(x_1\mid y_{1:H})
\end{equation}
and, for each $t=1,\ldots,H-1$,
\begin{equation}
q_\lambda^\star(x_{t+1}\mid x_{1:t},y_{1:H})
=
p(x_{t+1}\mid x_{1:t},y_{1:H}).
\label{eqn:vsmc-target-conditional}
\end{equation}
For the Markov state-space model above, this posterior conditional
reduces to
\begin{equation}
q_\lambda^\star(x_{t+1}\mid x_{1:t},y_{1:H})
=
p(x_{t+1}\mid x_t,y_{t+1:H}).
\end{equation}
\end{pro}

\begin{proof}
The variational bound can be written as
\begin{equation}
\mathcal L(\lambda)
=
\log p(y_{1:H})
-
\mathrm{KL}
\left(
q_\lambda(x_{1:H}\mid y_{1:H})
\,\Vert\,
p(x_{1:H}\mid y_{1:H})
\right).
\end{equation}
The posterior factorizes by the chain rule as
\begin{equation}
p(x_{1:H}\mid y_{1:H})
=
p(x_1\mid y_{1:H})
\prod_{t=1}^{H-1}
p(x_{t+1}\mid x_{1:t},y_{1:H}).
\end{equation}
Substituting this factorization and the factorization of
$q_\lambda(x_{1:H}\mid y_{1:H})$ into the KL gives
\begin{align}
\mathrm{KL}
\left(
q_\lambda(x_{1:H}\mid y_{1:H})
\,\Vert\,
p(x_{1:H}\mid y_{1:H})
\right)
\
&=
\mathrm{KL}
\left(
q_\lambda(x_1\mid y_{1:H})
\,\Vert\,
p(x_1\mid y_{1:H})
\right)
\\ \nonumber
&+
\sum_{t=1}^{H-1}
\mathbb E_{q_\lambda(x_{1:t}\mid y_{1:H})}
\left[
\mathrm{KL}
\left(
q_\lambda(\cdot\mid x_{1:t},y_{1:H})
\,\Vert\,
p(\cdot\mid x_{1:t},y_{1:H})
\right)
\right].
\end{align}
Each term is nonnegative. Hence the optimum is attained by setting each
proposal factor equal to the corresponding posterior conditional.

It remains to simplify this conditional in the Markov case. Under the
state-space factorization above, for $t=1,\ldots,H-1$,
\begin{equation}
p(x_{t+1}\mid x_{1:t},y_{1:H})
=
p(x_{t+1}\mid x_t,y_{t+1:H}),
\end{equation}
because, conditional on $x_t$, the distribution of $x_{t+1}$ and the
future observations $y_{t+1:H}$ is independent of the earlier latent
states $x_{1:t-1}$ and observations $y_{1:t}$. This proves the final
statement.
\end{proof}

Thus, following the variational interpretation of SMC in
\citet{NLR+18}, the ideal proposal in a Markov model is the smoothing
conditional distribution of the next latent state given the current
latent state and the remaining observations. This differs from a myopic
proposal that conditions only on the next observation. VSMC is therefore
useful when the proposal can learn to predict latent choices that are
favored by future evidence, a property used below when adapting VSMC to
posterior inference over deterministic policies.

\section{Probabilistic Model}
\label{sec:prob-model}

Much of the control-as-inference literature introduces auxiliary
\emph{optimality} variables or other fictitious observations in order to
cast planning into a standard latent--observed graphical model.  Here we
avoid such augmentation and instead work directly with an unnormalized
target distribution over the object of interest---the policy.  Posterior
inference only requires access to an unnormalized density $\tilde p(\pi)$
(up to a multiplicative constant), possibly through an unbiased
stochastic estimator of $\tilde p(\pi)$~\citep{AR09} or $\log \tilde
p(\pi)$~\citep{HBW+13}.

Since we are interested in inferring a policy, the \emph{policy} $\pi$ is
the latent random variable. The objective of MDP planning is to identify
policies that maximize expected return. To align the probabilistic model
with this objective, we assign to each policy an unnormalized
\textbf{probability of optimality} that is monotone in its expected
return.

Specifically, we define the unnormalized log probability of a policy
as the \emph{expected return} obtained by the agent following
the policy over a truncated rollout of length at most $H$, with
expectation over trajectories $\tau_\pi$ distributed according to
the dynamics induced by policy $\pi$:
\begin{equation}
\log \tilde p(\pi) = \mathbb{E}_{\tau_\pi} \sum_{t=1}^{H} R\left(s_t, a_t, s_{t+1}\right),
\label{eqn:logp_pi}
\end{equation}
where $a_t = \pi(s_t)$ and $s_{t+1}\sim \mathcal{T}(\cdot\mid s_t,a_t)$
for $t=1,\ldots,H$. Because goal states are absorbing, this sum
coincides with the episode return when the goal is reached
before $H$, and otherwise uses the rollout truncation horizon as
a computational cutoff.  This induces a Boltzmann--Gibbs
distribution over \emph{policies} \citep{ZMB+08,T09,L18}.
More formally, this distribution is defined relative to a reference
measure $\mu$ on the policy space:
\begin{equation}
P(d\pi)=Z^{-1}\exp(J(\pi))\,\mu(d\pi),
\qquad
Z=\int \exp(J(\pi))\,\mu(d\pi),
\end{equation}
where $J(\pi)$ denotes the expected truncated return in
Eq.~\eqref{eqn:logp_pi}. This reference measure is not used here as part
of a prior--likelihood decomposition with fictitious observations; it is
only the base measure with respect to which the Gibbs density is defined.
In the finite deterministic-policy setting used in this paper, $\mu$ may
be taken to be counting measure, equivalently the uniform measure over
deterministic policies. Extensions beyond finite discrete policy spaces
require specifying a proper reference measure on the chosen measurable
policy class. With finite horizon and rewards bounded above, $Z<\infty$
for any finite reference measure.

Note that neither actions nor states are treated as Bayesian random
variables for which a posterior is sought. Although the policy (if
stochastic policies are considered) induces a stochastic action
selection rule and the environment induces stochastic state transitions,
these are generative rather than inferential sources of randomness:
actions and state transitions are sampled forward from their respective
distributions rather than conditioned for the purpose of posterior
inference. The randomness they induce propagates into the estimation of
the unnormalized log probability of the policy, so that
$\log \tilde p(\pi)$ is available only through noisy Monte Carlo
evaluations---by computing the return of a single truncated
rollout:
\begin{equation}
\log \widehat{\tilde p}(\pi) = \sum_{t=1}^{H} R\left(s_{t}, a_t, s_{t+1}\right).
\label{eqn:hat_logp_pi}
\end{equation}
Stochastic estimate~\eqref{eqn:hat_logp_pi} lays the basis for posterior
inference of the policy distribution.

\subsection{Posterior-Induced Stochastic Policy}
\label{sec:prob-model-induced-policy}

The posterior over deterministic policies also induces a stochastic
action-selection rule by marginalization. For any state $s$, define
\begin{equation}
p^\star(a\mid s)
=
\Pr_{\pi\sim p(\cdot)}[\pi(s)=a]
=
\mathbb{E}_{\pi\sim p(\cdot)}
\left[
\mathbf{1}\{\pi(s)=a\}
\right],
\label{eqn:posterior-predictive-policy}
\end{equation}
where $p(\pi)\propto \tilde p(\pi)$ is the posterior over deterministic
policies induced by Eq.~\eqref{eqn:logp_pi}. We refer to
$p^\star(\cdot\mid s)$ as the \textbf{optimal stochastic policy under
preference uncertainty}. This is a notion introduced in this paper: it
denotes the stochastic policy obtained by marginalizing the posterior
over deterministic policies under the model in Eq.~\eqref{eqn:logp_pi}.
Its action probabilities are therefore the posterior probabilities that
the corresponding actions are prescribed by posterior-supported
deterministic policies at state $s$.
Stochasticity is thus not introduced through entropy regularization;
it arises from posterior uncertainty over which deterministic policy
is preferred by the model.


\subsection{Worked Tabular Examples}
\label{sec:worked-tabular-examples}

We use two small tabular examples to clarify the relationship between the
policy-posterior construction, entropy-regularized control, and
optimality-variable control-as-inference.  The first example shows a setting in
which the policy-posterior construction and entropy-regularized control
coincide, while optimality-variable control-as-inference differs.  The second
example shows that the coincidence does not hold in general: when a state can
be revisited, a posterior over deterministic policies exhibits policy-level
commitment, whereas an entropy-regularized stochastic policy resamples actions
on each visit.  Throughout this subsection the entropy coefficient is set to
one.

\paragraph{One-step stochastic bandit.}

Consider a one-step decision problem with initial state $s_3$, actions
$a_1,a_2$, and terminal states $s_1,s_2$.  Action $a_i$ moves deterministically
to $s_i$, where a Bernoulli reward is obtained:
\begin{equation}
R_i =
\begin{cases}
1, & \text{with probability } r_i,\\
0, & \text{with probability } 1-r_i .
\end{cases}
\end{equation}
Thus the expected return of choosing $a_i$ is
\begin{equation}
J(a_i)=\mathbb{E}[R_i]=r_i .
\end{equation}
Let $\pi_i$ denote the deterministic policy that chooses $a_i$ at $s_3$.
Under the policy-posterior construction, the unnormalized density assigned to
$\pi_i$ is
\begin{equation}
\tilde p(\pi_i)=\exp{J(\pi_i)}=\exp(r_i).
\end{equation}
The posterior-induced action probabilities are therefore
\begin{equation}
p^\star(a_1\mid s_3)
=
\frac{\exp(r_1)}
{\exp(r_1)+\exp(r_2)}
\end{equation}
and
\begin{equation}
p^\star(a_2\mid s_3)
=
\frac{\exp(r_2)}
{\exp(r_1)+\exp(r_2)} .
\end{equation}

Entropy-regularized RL gives the same probabilities in this one-step setting.
Writing $p$ for the probability of choosing $a_1$, the entropy-regularized
objective is
\begin{equation}
\max_{p\in[0,1]}
\left[
p r_1 + (1-p)r_2
-
p\log p
-
(1-p)\log(1-p)
\right].
\end{equation}
The optimizer is
\begin{equation}
p_{\mathrm{ER}}(a_1\mid s_3)
=
\frac{\exp(r_1)}
{\exp(r_1)+\exp(r_2)}
=
p^\star(a_1\mid s_3),
\end{equation}
and similarly for $a_2$.  Thus, in a one-step bandit, the
policy-posterior construction and entropy-regularized RL coincide.

Optimality-variable control-as-inference differs in the presence of stochastic
rewards.  In that construction, the unnormalized weight of action $a_i$ is
obtained by exponentiating the realized reward and then taking expectation:
\begin{equation}
\mathbb{E}\left[\exp(R_i)\right]
=
r_i\exp(1)+(1-r_i)\exp(0)
=
r_i e+(1-r_i).
\end{equation}
Hence the corresponding action probability is
\begin{equation}
p_{\mathrm{CAI}}(a_1\mid s_3)
=
\frac{r_1e+(1-r_1)}
{r_1e+(1-r_1)+r_2e+(1-r_2)} .
\end{equation}
This is generally different from
\begin{equation}
\frac{\exp(r_1)}
{\exp(r_1)+\exp(r_2)}.
\end{equation}
The one-step bandit therefore separates the policy-posterior and
entropy-regularized objectives from optimality-variable
control-as-inference, while also showing that the former two need not differ in
the simplest setting.

The same calculation applies to any finite one-step bandit: for actions
$a_1,\ldots,a_K$ with expected rewards $\mu_i=\mathbb{E}[R_i]$, both the
policy-posterior construction and entropy-regularized RL assign
\begin{equation}
    p(a_i)
    =
    \frac{\exp(\mu_i)}
         {\sum_{j=1}^K \exp(\mu_j)} .
\end{equation}

\paragraph{Stochastic continuation.}

We next give an example in which the policy-posterior construction
and entropy-regularized RL differ.  Let $s_1$ be the initial state and let
$s_2$ be terminal.  There are two actions at $s_1$.  Action $a_1$ gives zero
reward and either returns to $s_1$ or terminates:
\begin{equation}
\mathcal{T}(s_1 \mid s_1,a_1) = \rho, \qquad \mathcal{T}(s_2\mid
s_1, a_1) = 1-\rho,\qquad R(s_1,a_1,s_1)=R(s_1,a_1,s_2)=0.
\end{equation}
where $0\leq \rho<1$.  Action $a_2$ terminates immediately and gives reward
one:
\begin{equation}
\mathcal{T}(s_2\mid s_1,a_2)=1,
\qquad
R(s_1,a_2,s_2)=1.
\end{equation}

There are two deterministic policies at $s_1$:
\begin{equation}
\pi_1(s_1)=a_1,
\qquad
\pi_2(s_1)=a_2.
\end{equation}
Under $\pi_1$, every realized trajectory receives total reward zero.  The
transition probability $\rho$ affects the number of visits to $s_1$, but it
does not affect the total reward.  Thus
\begin{equation}
J(\pi_1)=0.
\end{equation}
Under $\pi_2$, the process terminates immediately with reward one, so
\begin{equation}
J(\pi_2)=1.
\end{equation}
Thus
\begin{equation}
\tilde p(\pi_1)=\exp(0)=1,
\qquad
\tilde p(\pi_2)=\exp(1)=e.
\end{equation}
Therefore the posterior-induced action probabilities are
\begin{equation}
p^\star(a_1\mid s_1)
=
\Pr\nolimits_{\pi\sim p}[\pi(s_1)=a_1]
=
\frac{1}{1+e},
\end{equation}
and
\begin{equation}
p^\star(a_2\mid s_1)
=
\Pr\nolimits_{\pi\sim p}[\pi(s_1)=a_2]
=
\frac{e}{1+e}.
\end{equation}
These probabilities are independent of $\rho$.

Entropy-regularized RL gives a different stationary policy.  Let
$V_{\mathrm{ER}}(s_2)=0$, and write
\begin{equation}
V := V_{\mathrm{ER}}(s_1).
\end{equation}
The soft Bellman equation is
\begin{equation}
V
=
\log\left(
\exp(\rho V)
+
\exp(1)
\right),
\end{equation}
because the soft value after taking $a_1$ is returned to with probability
$\rho$, whereas $a_2$ terminates with reward one.  Let
\begin{equation}
x:=\exp(V).
\end{equation}
Then $x$ satisfies
\begin{equation}
x=x^\rho+e.
\end{equation}
The entropy-regularized action probabilities are
\begin{equation}
p_{\mathrm{ER}}(a_2\mid s_1)
=
\frac{\exp(1)}
{\exp(\rho V)+\exp(1)}
=
\frac{e}{x},
\end{equation}
and
\begin{equation}
p_{\mathrm{ER}}(a_1\mid s_1)
=
\frac{\exp(\rho V)}
{\exp(\rho V)+\exp(1)}
=
\frac{x^\rho}{x}.
\end{equation}
For $\rho=0$, the example reduces to the one-step case and the two policies
coincide.  For every $\rho>0$, however, $x^\rho>1$ and hence
\begin{equation}
x=x^\rho+e>1+e.
\end{equation}
Consequently,
\begin{equation}
p_{\mathrm{ER}}(a_2\mid s_1)
=
\frac{e}{x}
<
\frac{e}{1+e}
=
p^\star(a_2\mid s_1).
\end{equation}

For example, when $\rho=1/2$, the equation
\begin{equation}
x=x^{1/2}+e
\end{equation}
has solution
\begin{equation}
x=
\left(
\frac{1+\sqrt{1+4e}}{2}
\right)^2.
\end{equation}
Thus
\begin{equation}
p_{\mathrm{ER}}(a_2\mid s_1)
=
\frac{e}{
\left(
\frac{1+\sqrt{1+4e}}{2}
\right)^2}
\approx 0.55,
\end{equation}
whereas
\begin{equation}
p^\star(a_2\mid s_1)
=
\frac{e}{1+e}
\approx 0.73.
\end{equation}

This example isolates the relevant distinction.  The policy-posterior
construction places a posterior distribution over deterministic policies.  A
sampled policy that chooses $a_1$ at $s_1$ is therefore committed to choosing
$a_1$ whenever $s_1$ is revisited.  Entropy-regularized RL instead optimizes a
single stochastic policy.  If that policy chooses $a_1$ and the process returns
to $s_1$, the next action is resampled.  The difference is therefore one of
policy-level commitment versus action-level resampling.

\section{Inference}
\label{sec:inference}

We perform posterior inference over \textbf{deterministic
policies}. The formulation can also be applied to stochastic-policy
classes once a reference measure on the corresponding policy space is
specified, but we focus on deterministic policies because, for the
finite-horizon MDP objective considered here, there exists an optimal
deterministic policy. In our formulation, stochasticity in control is
therefore not introduced by treating the policy itself as intrinsically
stochastic; instead, it arises from posterior uncertainty over which
deterministic policy is optimal. This keeps uncertainty at the policy
level and avoids adding a separate layer of action-level randomness.
The stochastic control rule is the posterior-induced policy
$p^\star(a\mid s)$ defined in
Section~\ref{sec:prob-model-induced-policy}; in practice, we use the
variational approximation $q(a\mid s)$ and sample actions from
$q(\cdot\mid s)$ at execution time.

A natural baseline for posterior inference under the rollout estimator
in~\eqref{eqn:hat_logp_pi} is structured variational inference
\citep{HBW+13}, but single-trajectory objectives are often underdispersed
and prone to mode collapse~\citep{BKM17,YVS+18}. The sequential structure of returns makes
variational sequential Monte Carlo (VSMC) a natural alternative: its
multi-particle objective gives a tighter variational bound and a more
robust particle approximation. In our setting, the expected return
$J(\pi)$ is available only through the unbiased Monte Carlo
estimator~\eqref{eqn:hat_logp_pi}. This estimator can be treated as
exogenous estimator noise on the target density, analogous to
pseudo-marginal inference~\citep{AR09}, and incorporated directly into
VSMC optimization.

We assume a countable state space and a finite action space to enable
revisit bookkeeping and categorical action proposals; these assumptions
simplify the inference mechanics and do not affect the probabilistic
model. For deterministic policy inference, a policy assigns a single
action to each visited state, sampled on first visit:
\begin{equation}
\pi(s) \equiv a \mid s \sim \mathrm{Categorical}(\pmb{p}(s)).
\label{eqn:deterministic-policy}
\end{equation}
In our case studies, $\pmb{p}(s)$ is parameterized by a neural network
over a factored representation of $s$.

\paragraph{SMC sweep.}
Two adjustments to the vanilla SMC sweep are required:

\begin{enumerate}
\item \textbf{Deterministic policy consistency.} For each particle, the
action for a state is sampled from the proposal only on the first visit
to that state and is reused on all revisits. Equivalently, the particle
memoizes $\pi(s)$; on revisits, the proposal assigns probability $1$ to
the previously sampled action and $0$ to all others.

\item \textbf{Coupled transition randomness.} To ensure particle weights
reflect policy differences rather than independent realizations of
environment noise, transition randomness is shared across particles
within a sweep. Specifically, if two particles visit the same state $s$
and take the same action $a$ on the same visit count $k$, they are forced
to transition to the same successor state $s'$. This can be implemented
by lazily sampling and caching $\hat T^{k}_{s,a}\sim \mathcal{T}(\cdot\mid s,a)$
once per queried $(s,a,k)$ and reusing it for all particles in the sweep,
so that inference proceeds under a shared random realization $\hat T$ of
the dynamics.\footnote{This is related to common random numbers
\citep{MRF+20}, but is used here to preserve policy-level comparisons
across particles rather than only for variance reduction.}
\end{enumerate}

With these modifications, each sweep proceeds as in standard
SMC: particles advance under $\mathsc{Step}$, update weights,
and resample as needed.  Full pseudocode appears in the
appendix. Proposition~\ref{pro:policy-vsmc-target} specializes
Proposition~\ref{pro:vsmc-target-proposal} to policy inference.

\begin{pro}[Policy-VSMC target]
\label{pro:policy-vsmc-target}
Let $A_s=\pi(s)$ denote the latent action assignment of a deterministic
policy at state $s$. Under the posterior $p(\pi)$ induced by
Eq.~\eqref{eqn:logp_pi}, the ideal sequential proposal for a first visit
to state $s_t$ is
\begin{equation}
q^\star(a\mid s_t,\pi_{<t})
=
\Pr
\left[
A_{s_t}=a
\mid
\pi \text{ agrees with } \pi_{<t}
\right],
\label{eqn:policy-vsmc-conditional-target}
\end{equation}
where $\pi_{<t}$ denotes the action assignments fixed on previous first
visits. Marginalizing over the previous assignments gives the
posterior-induced stochastic policy
\begin{equation}
p^\star(a\mid s)
=
\Pr[\pi(s)=a].
\label{eqn:policy-vsmc-marginal-target}
\end{equation}
Thus the learned state-indexed proposal provides an amortized
(shared across SMC sweeps and state visits) approximation to the
conditional proposal during SMC and to the marginal
$p^\star(\cdot\mid s)$ at execution time.
\end{pro}

\begin{proof}
A deterministic policy is equivalently the collection of latent action
assignments $\{A_s:s\in\mathcal S\}$. By the first-visit construction
above, Policy VSMC samples such an assignment only on the first visit to
$s$ and reuses it thereafter. Hence the sequential latent variables are
the first-visit policy assignments. Applying
Proposition~\ref{pro:vsmc-target-proposal} to this sequential
representation gives the posterior conditional in
Eq.~\eqref{eqn:policy-vsmc-conditional-target}. The Markov property of
the MDP ensures that, once the current state and the previously fixed
policy assignments are given, future rewards depend on the past trajectory
only through these quantities. Thus the next proposal variable is the
policy assignment $A_{s_t}$, not an action indexed by the full history.

The posterior-induced stochastic policy is the marginal probability that a
posterior-sampled deterministic policy assigns action $a$ to state $s$.
Marginalizing the conditional proposal over previous first-visit
assignments therefore gives
\begin{equation}
\Pr[A_s=a]
=
\Pr[\pi(s)=a]
=
p^\star(a\mid s),
\end{equation}
which proves the claim.
\end{proof}

At execution time, we use the learned proposal as an amortized approximation
to the posterior-induced stochastic policy. If the VSMC approximation were
exact, sampling $\pi$ and executing $\pi(s)$ would sample from
$p^\star(\cdot\mid s)$. The learned proposal $q_\theta(a\mid s)$ is trained
to approximate the corresponding posterior action conditionals, so we execute
by sampling from $q_\theta(\cdot\mid s)$ as an approximation to
$p^\star(\cdot\mid s)$.

\paragraph{Optimization objective} The original VSMC formulation assumes reparameterizable proposals and
typically omits score-function terms associated with the non-differentiable
\emph{resampling} operation due to their high variance. In finite-action
MDPs, however, the categorical proposal over actions is not
reparameterizable. Consequently, while the resampling-induced score terms
may still be dropped, the score-function contribution from
\emph{sampling actions from the proposal} must be retained to obtain
meaningful gradients. Using a temporally stratified, variance-reduced
learning signal~\citep{SHW+15}, we optimize
\begin{equation}
\mathcal{L} = \log \hat Z + \sum_{t=1}^H \left(\overline{\log \hat
Z_t} \cdot \sum_{i=1}^N \log q_\theta(a_{t,i} \mid s_{t, i})\right),
\label{eqn:det-surrogate-objective-stratified}
\end{equation}
where $\log \hat Z_t$ denotes the contribution of steps $t,\ldots,H$ to
$\log \hat Z$, and the overline denotes a stop-gradient operation.

Proposition~\ref{pro:unbiased-gradient} shows that optimizing the
surrogate objective corresponds to stochastic gradient ascent on
a well-defined scalar objective, rather than a heuristic update
rule.

\begin{pro}[Score-function gradient with stopped resampling]
\label{pro:unbiased-gradient}
Let $\hat Z(\hat T, \mathbf a, \mathbf r)$ denote the SMC normalizing
constant estimator produced by one sweep of the procedure above, where
$\hat T$ is the shared random realization of the transition dynamics
induced by the coupled-transition rule,
$\mathbf a=\{a_{t,i}\}_{t=1,i=1}^{H,N}$ are the actions sampled from the
proposal $q_\theta$, and $\mathbf r$ denotes the resampling randomness.
Define
\begin{equation}
\mathcal{J}(\theta)
=
\mathbb{E}_{\hat T,\;\mathbf a \sim q_\theta,\;\mathbf r}
\left[
    \log \hat Z(\hat T, \mathbf a, \mathbf r)
\right].
\end{equation}
Then, treating resampling decisions as stopped-gradient randomness, the
gradient of the surrogate objective in
Eq.~\eqref{eqn:det-surrogate-objective-stratified} is an unbiased estimator
of the gradient of this stopped-resampling objective with respect to the
proposal parameters $\theta$.
\end{pro}

\begin{proof}
The SMC sweep contains three sources of randomness: the shared transition
realization $\hat T$, the proposal-sampled actions $\mathbf a$, and the
resampling decisions $\mathbf r$. As in standard VSMC, we do not
differentiate through the discrete resampling operation and omit the
corresponding resampling score-function terms. The claim is therefore about
the objective obtained after this standard VSMC approximation, with
gradients taken through the action proposal.

Conditioned on $\hat T$ and $\mathbf r$, the remaining $\theta$-dependent
randomness in $\log \hat Z(\hat T,\mathbf a,\mathbf r)$ comes from sampling
actions from $q_\theta$. The realized value of $\log \hat Z$ is then a
deterministic function of the sampled actions, the shared transition
realization, and the resampling decisions. Applying the score-function identity
to the proposal-sampled actions gives
\begin{equation}
\nabla_\theta
\mathbb{E}_{\mathbf a}
\left[
    \log \hat Z(\hat T,\mathbf a,\mathbf r)
\right]
=
\mathbb{E}_{\mathbf a}
\left[
    \log \hat Z(\hat T,\mathbf a,\mathbf r)
    \sum_{t,i}
    \nabla_\theta \log q_\theta(a_{t,i}\mid s_{t,i})
\right],
\end{equation}
where the expectation is conditional on $\hat T$ and $\mathbf r$.
Introducing a stop-gradient baseline preserves unbiasedness of the
score-function estimator. Taking expectation over $\hat T$ and $\mathbf r$
completes the result.

The temporally stratified signal $\log \hat Z_t$ in
Eq.~\eqref{eqn:det-surrogate-objective-stratified} is a variance-reduction,
or Rao--Blackwellization, of the same score-function estimator: each sampled
action is paired only with the future contribution to $\log \hat Z$ that can
depend on it. This leaves the expectation unchanged.
\end{proof}

\section{Related Work}

This work relates to (i) probabilistic formulations of planning and
control, and (ii) entropy-regularized reinforcement learning. Our key
distinction is that inference is carried out over \emph{policies}
themselves, rather than over trajectory-level auxiliary variables or
within a directly optimized parametric stochastic policy.

\subsection{Control and Planning as Inference}

Casting control and planning as probabilistic inference has a long
history, including planning-as-inference in graphical models~\citep{A03}
and subsequent formulations that encode optimality via auxiliary
variables or likelihood terms that bias trajectories toward high return
(e.g.,~\citet{BT12,TS06}).  Active inference likewise casts action
selection as approximate Bayesian inference under a generative model
with preferences over outcomes.  More recently, control-as-inference
derivations have shown that entropy-regularized RL objectives arise from
variational inference constructions (e.g.,~\citet{L18}), typically by
introducing fictitious observations or optimality variables.

We adopt this perspective but make a different modeling choice: the
\emph{policy itself} is the latent random variable, and its expected
return defines an unnormalized log density. This yields a posterior over
policies directly, without introducing additional observation channels
or trajectory-level optimality variables.

\subsection{Entropy-Regularized Reinforcement Learning}

Entropy-regularized RL and stochastic policy optimization methods,
including policy gradients~\citep{SMS+99}, soft Q-learning~\citep{HTA+17},
and Soft Actor-Critic (SAC)~\citep{HZA+18,C19}, optimize objectives of the form
$\mathbb{E}[R(\pi)] + \alpha \mathcal{H}(\pi)$.  Connections and
equivalences among these approaches have been studied
extensively (e.g., \citet{XCA17}), and from a
control-as-inference viewpoint entropy can be interpreted as
arising from a variational bound.

While related, our approach differs in two respects. First,
stochasticity reflects posterior uncertainty over deterministic
policies, rather than entropy within a single learned stochastic
policy. Second, this inference formulation makes
\textbf{variational sequential Monte Carlo (VSMC)}~\citep{NLR+18} a
natural approximation method, in contrast to the single-trajectory
variational objectives that underlie many entropy-regularized RL
algorithms.

\subsection{Risk-sensitive and soft-robust MDPs}

Risk-sensitive and soft-robust MDPs provide a related but distinct comparison
point. A common risk-sensitive objective is the entropic risk measure of
cumulative cost,
\begin{equation}
    \log
    \mathbb{E}_{T,\pi}
    \left[
        \exp\left(
            \alpha\sum_{t=1}^H c_t
        \right)
    \right],
\end{equation}
which has recently been connected to soft-robust MDPs and studied
algorithmically~\citep{ZHL24,MBG+25}. With rewards rather than costs, the
analogous exponentiated-return objective coincides with the optimality-variable
control-as-inference objective~\citep{L18}:
\begin{equation}
    \log
    \mathbb{E}_{T,\pi}
    \left[
        \exp\left(
            \sum_{t=1}^H R_t
        \right)
    \right].
\end{equation}

The distinction from our policy-posterior construction is the order in which
environment stochasticity and policy uncertainty are handled. Entropic-risk and
optimality-variable control-as-inference objectives exponentiate realized
trajectory returns and then marginalize over the stochastic quantities,
\textit{effectively treating stochastic transitions as latent model variables}.
In contrast, our construction first averages environment stochasticity when
evaluating a fixed deterministic policy,
\begin{equation}
    J(\pi)
    =
    \mathbb{E}_{T}
    \left[
        \sum_{t=1}^H R_t
    \right],
\end{equation}
and then assigns policy density $\tilde p(\pi)=\exp\{J(\pi)\}$. The posterior
stochastic policy is obtained by marginalizing over deterministic policies.  As
the worked tabular example shows, these constructions can coincide in simple
deterministic settings but generally differ in stochastic MDPs.  Thus,
entropic-risk methods intentionally modify the control objective to control
sensitivity to high-cost or high-reward outcomes. Our objective is different:
we retain expected-return evaluation for each deterministic policy and use
inference to represent uncertainty over which deterministic policy is selected.

Risk-sensitive RL is therefore not a direct algorithmic baseline for the
experiments in this paper. Its purpose is to change the criterion used to
evaluate a policy under environment uncertainty, for example by making the
agent optimistic or pessimistic with respect to stochastic transitions and
rewards. This work instead assumes that the environment model is fixed and uses
a posterior distribution over policies to represent uncertainty over agent
preferences.

\section{Experiments}

We evaluate the proposed policy inference framework across a range of
domains designed to expose different structural aspects of decision
making under uncertainty, and compare it to entropy-regularized policy
optimization. Throughout the experiments we contrast inference over
distributions of deterministic policies (VSMC) with direct optimization
of entropy-regularized stochastic policies (SAC). The experiments are
designed to examine (i) qualitative structure of induced behavior in a
diagnostic domain, and (ii) differences between deterministic-policy
inference and entropy-regularized optimization across increasingly
stochastic and complex planning problems. The purpose of this
comparison is not to establish uniform superiority in expected reward,
since the two methods optimize different objectives, but to examine how
these objectives induce different solution structure, uncertainty
representations, and action-selection behavior across domains.
The experiments are intended to probe the consequences of this modeling
choice---inference over policies---rather than to position VSMC as a new
state-of-the-art optimizer for entropy-regularized control benchmarks.

We begin by exploring the proposed policy inference framework on a grid
world domain. The ease of static visualization and apparent simplicity
of grid worlds facilitates qualitative inspection of the induced
stochastic policy. We then use the grid worlds and three standard
discrete benchmarks from the literature---Blackjack, Triangle Tireworld,
and Academic Advising---to compare policy VSMC to discrete Soft
Actor-Critic (SAC)\footnote{SAC optimizes an
entropy-regularized expected-return objective, whereas policy
VSMC maximizes an SMC log-evidence bound $\mathbb{E}[\log \hat
Z]$, where $\log \hat Z$ aggregates particle weights via a
log-mean-exp across the sweep. The two objectives are not the
same, so differences in performance are informative primarily as
differences in induced behavior rather than as a pure leaderboard
comparison.}~\citep{HZA+18,C19}, highlighting differences in the
resulting policies and their suitability to particular MDP types.

Throughout the experiments, we run VSMC with 10 particles for $50{,}000$
iterations (SMC sweeps), adjusting the initial learning rate between
$10^{-5}$ and $3 \cdot 10^{-4}$ for each domain, with cosine decay to
$0.1$ of the initial rate. These settings are reported to support
reproducibility rather than to suggest an ``optimal'' tuning: in our
experiments VSMC is stable across a relatively broad range of learning
rates and schedules, the iteration budget is chosen to be large enough
to ensure convergence across all domains, and we intentionally keep the
number of particles small. This is consistent both with the original
VSMC results of \citet{NLR+18}, which show strong performance with
small particle counts, and with the broader particle-variational-
inference analysis of \citet{RKL+18}, which shows that increasing the
number of particles can tighten the bound while degrading the
signal-to-noise ratio of inference-network gradients. We adapt SAC from CleanRL~\citep{HDY+22} for
discrete actions by using MLP critics with two hidden layers, and train
for $1{,}000{,}000$ time steps. In both VSMC and SAC, all networks use
two hidden layers of width 64. Each algorithm--domain pair is evaluated
over 25 independent training runs, and policies are evaluated using
$10{,}000$ trajectories per run.  The entropy weight for SAC is
kept at $\alpha=1$ except where varied explicitly. We fix
$\alpha$ rather than tuning a domain-specific entropy-weight
schedule because the experiments are intended to compare
modeling choices---posterior inference over a distribution of
deterministic policies versus direct optimization of a
stochastic policy---rather than optimize SAC hyperparameters for
each benchmark.

\subsection{Grid Worlds}
 
Grid worlds provide a controlled setting in which policy
distributions can be visualized directly, allowing qualitative
inspection of multimodality, uncertainty, and variability across
runs. They therefore serve as a diagnostic domain for understanding
the behavior of the inference procedure itself.
 
In the grid world domain we use in this study, the environment
is a rectangular grid. Four actions---Right, Up, Down, and
Left---advance the agent to the corresponding adjacent cell.  An
action that would take the agent out of the grid has no effect.
The environment is slippery: upon any action the agent moves,
with probability $p_\mathit{succ}=0.8$, in the direction of the
action, and otherwise in an adjacent direction. The reward
collected by the agent upon each action is determined by the
color of the cell to which the agent moves: grey (pavement) ---
0, red (gravel) --- -1, yellow (goal) --- +5, green (swamp) ---
-5. Each step incurs a cost of -0.1. Yellow and green cells are
absorbing: once the agent reaches a yellow or a green cell, no
further reward is collected and no action moves the agent out of
the cell.

Inferred policies are represented by \textit{policy and
occupancy maps} in the figures below
(Figure~\ref{fig:gw}--\ref{fig:gw-vsmc-sac}).  A black border denotes
the initial cell.  The darkness of the middle part of a cell
represents the occupancy: the darker a cell, the more
trajectories passed through the cell. The white cross with
unevenly sized beams in the center of each cell represents the
policy distribution in that cell, with the length of each beam
denoting the probability of the corresponding direction. The
maps are averaged over 25 runs: blurrier crosses mean more
variation in the inferred policies across runs.

\begin{figure}
\centering
\begin{subfigure}[t]{0.32\textwidth}
  \centering
  \includegraphics[width=0.8\linewidth]{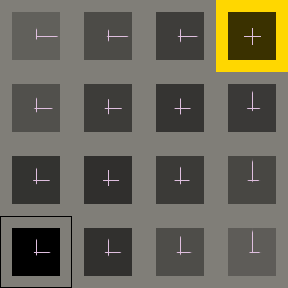} 
  \caption{flat world}
  \label{fig:gw-2}
\end{subfigure}
\begin{subfigure}[t]{0.32\textwidth}
  \centering
  \includegraphics[width=0.8\linewidth]{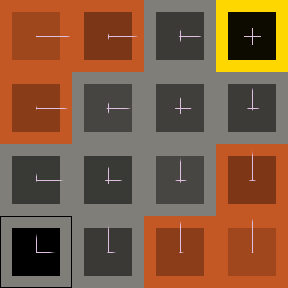}
  \caption{unimodal world}
  \label{fig:gw-4}
\end{subfigure}
\begin{subfigure}[t]{0.32\textwidth}
  \centering
  \includegraphics[width=0.8\linewidth]{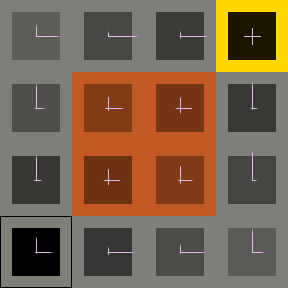} 
  \caption{multimodal world}
  \label{fig:gw-3}
\end{subfigure}
\caption{Grid worlds: policies and occupancies}
\label{fig:gw}
\end{figure}
We begin by applying policy VSMC to three $4\times 4$ grid
worlds (Figure~\ref{fig:gw}), using a rollout horizon of 20 steps.
In the flat world
(Figure~\ref{fig:gw-2}), trajectories
cover the grid evenly. In the ``unimodal'' world
(Figure~\ref{fig:gw-4}) the expensive to travel red regions 
in the complementary corners of the grid push trajectories 
closer to the diagonal connecting the starting and the goal
cells. In the ``multimodal'' world (Figure~\ref{fig:gw-3}) the red
region is in the center of the grid, pushing the trajectories to
pass along the edges of the grid. Because a policy
\textit{distribution} is inferred (rather than just a single
policy with the highest expected return), multiple actions in
each cell have non-zero probabilities.

\begin{figure}
\centering
\begin{subfigure}[c]{0.26\textwidth}
\centering
 \includegraphics[width=\linewidth]{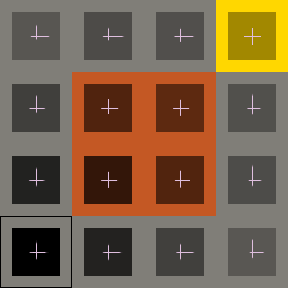} 
 \caption{deterministic policies\\ not enforced}
 \label{fig:gw-3-ndp}
\end{subfigure}
\hspace{6pt}
\begin{subfigure}[c]{0.16\textwidth}
\centering
 \includegraphics[width=\linewidth]{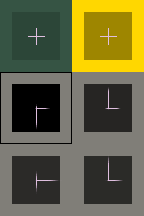} 
 \caption{shared\\ dynamics}
 \label{fig:gw-5-visits}
\end{subfigure}
\hspace{2pt}
\begin{subfigure}[c]{0.16\textwidth}
\centering
 \includegraphics[width=\linewidth]{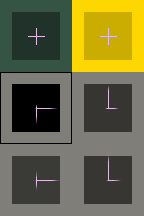}
 \caption{independent\\ dynamics}
 \label{fig:gw-5-nsd-visits}
\end{subfigure}
\hspace{6pt}
\begin{subfigure}[c]{0.24\textwidth}
\centering
\includegraphics[width=\textwidth]{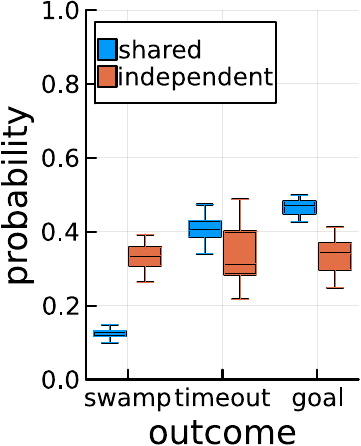}
\caption{shared vs. independent}
\label{fig:gw-5-nsd-outcomes}
\end{subfigure}
\caption{Grid worlds: ablation studies}
\label{fig:ablation}
\end{figure}
For policy inference, the SMC sweep of VSMC was modified to a)
enforce deterministic policies b) share the same environment
dynamics among all particles. Figure~\ref{fig:gw-3-ndp} 
shows the effect of dropping the enforcement of deterministic
policies: a higher-entropy policy distribution.
Figures~\ref{fig:gw-5-visits}--\ref{fig:gw-5-nsd-outcomes} explore the
effect of sharing environment dynamics among all particles on
the inferred policies. A small slippery instance with two
absorbing cells, a swamp with reward of -5 at (1, 3) and a goal
with reward of 5 at (2, 3), and $p_\mathit{succ}=0.5$ is used
for the comparison.  The starting position is at (1, 2) and the
rollout horizon is 10 steps. To avoid slipping into the
swamp, the agent should move \textit{down}, to (1, 1), and
this is what the policy with shared dynamics mostly suggests.
With independent dynamics the agent frequently moves
\textit{right}, along the shortest path to the goal.

Finally, we compare VSMC to SAC. Figure~\ref{fig:gw-vsmc-sac}
shows the policies inferred by VSMC (Figure~\ref{fig:gw-3-vsmc})
and SAC (Figure~\ref{fig:gw-3-sac}) and compares their
trajectory return distributions (Figure~\ref{fig:gw-3-vsmc-sac}). VSMC and
SAC trajectory return distributions are close but different, and the
policies differ in particular along the grid edges --- SAC,
optimizing an entropy-regularized stochastic policy, uses
actions directed toward the grid boundaries to increase the
entropy. VSMC penalizes such actions strongly because a
deterministic policy directing the agent into a grid boundary
can escape the current cell only due to environment
stochasticity.
\begin{figure}
\centering
\begin{subfigure}[t]{0.32\textwidth}
  \centering
  \includegraphics[width=0.8\linewidth]{figures/gw-3-visits.png}
  \caption{VSMC}
  \label{fig:gw-3-vsmc}
\end{subfigure}
\begin{subfigure}[t]{0.32\textwidth}
  \centering
  \includegraphics[width=0.8\linewidth]{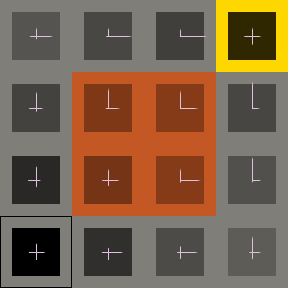}
  \caption{SAC}
  \label{fig:gw-3-sac}
\end{subfigure}
\begin{subfigure}[t]{0.32\textwidth}
\centering
\includegraphics[width=0.8\textwidth]{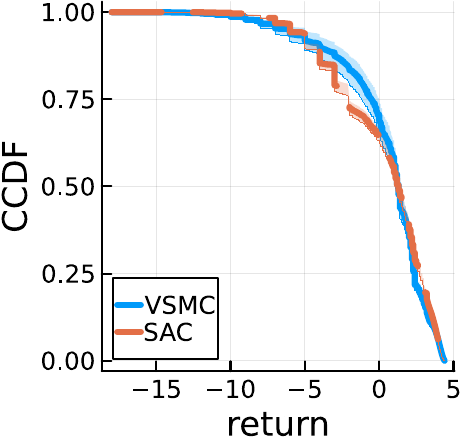}
\caption{trajectory return distributions}
\label{fig:gw-3-vsmc-sac}
\end{subfigure}
\caption{Grid worlds: VSMC vs. SAC}
\label{fig:gw-vsmc-sac}
\end{figure}

\subsection{Blackjack}

Blackjack is a card game where the goal is to beat
the dealer by obtaining cards that sum to closer to 21 (without going over 21)
than the dealer's cards.  Blackjack provides a stochastic control
problem with a compact state space and a known optimal policy,
making it possible to compare inferred policies against a
ground-truth solution while examining the effect of entropy
regularization in a domain where randomness arises primarily
from the environment rather than exploration. The game starts
with the dealer having one face up and one face down card, while
the player has two face up cards. All cards are drawn from an
infinite deck (i.e. with replacement).  The player has the sum
of cards held. The player can request additional cards (hit)
until they decide to stop (stick) or exceed 21 (bust). After the
player sticks, the dealer reveals their facedown card, and draws
cards until their sum is 17 or greater.  If the dealer goes
bust, the player wins.  If neither the player nor the dealer
busts, the outcome (win, lose, draw) is decided by whose sum is
closer to 21. 

\begin{table}
\centering
\caption{Blackjack: policy statistics}
\label{tab:bj}
\begin{tabular}{l | c c c c }
\textbf{policy} & \textbf{expected return} &   \multicolumn{3}{c}{\textbf{probability}} \\
                &                          & \textbf{loss} & \textbf{draw} & \textbf{win} \\ \hline
   VSMC & $-0.19 \pm 0.04$ & $0.57 \pm 0.02$ & $0.05 \pm 0.01$ & $0.38 \pm 0.02$ \\
      VSMC, $10 \cdot r$ & $-0.14 \pm 0.03$ & $0.54 \pm 0.01$ & $0.05 \pm 0.01$ & $0.41 \pm 0.01$ \\
     VSMC, $100 \cdot r$ & $-0.12 \pm 0.04$ & $0.53 \pm 0.02$ & $0.06 \pm 0.01$ & $0.41 \pm 0.02$ \\
             SAC & $-0.32 \pm 0.01$ & $0.63 \pm 0.01$ & $0.06 \pm 0.01$ & $0.31 \pm 0.01$ \\
      SAC, $\alpha=0.1$ & $-0.15 \pm 0.01$ & $0.54 \pm 0.02$ & $0.07 \pm 0.01$ & $0.39 \pm 0.02$ \\
     SAC, $\alpha=0.01$ & $-0.08 \pm 0.03$ & $0.50 \pm 0.05$ & $0.08 \pm 0.02$ & $0.42 \pm 0.03$ \\
         optimal & $-0.06$ & $0.49$ & $0.09$ & $0.42$
\end{tabular}
\end{table}

\begin{figure}
  \centering
  \includegraphics[width=0.8\linewidth]{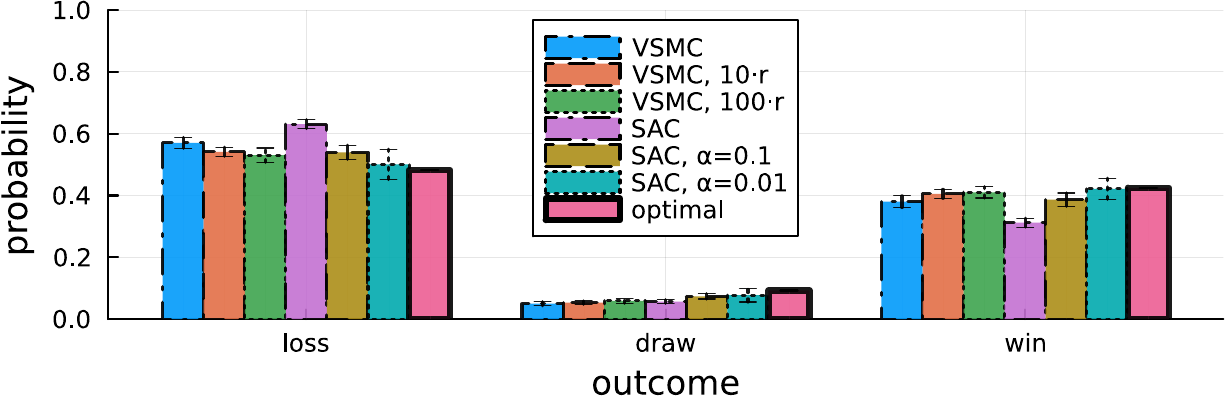}
  \caption{Blackjack: outcome probabilities}
  \label{fig:bj}
\end{figure}

The variant described in ~\citet[Example 5.1]{SB98}, as implemented
in Gymnasium~\citep{TKT+24}, is used in this study.  VSMC and SAC are compared to each
other and to the optimal policy, as a baseline, in Table~\ref{tab:bj} and Figure~\ref{fig:bj}.
The optimal (return-maximizing) player's policy was
computed by value iteration and the policy's statistics were estimated by
Monte-Carlo evaluation. Neither VSMC nor SAC with the default
temperature/entropy weight matches the optimal policy, which is expected
because both methods are deliberately regularized. The main observation
is not that one method dominates the other in mean return, but that the
two objectives induce different operating points: VSMC produces a lower
draw probability than SAC at $\alpha=1$, while matching this behavior in
SAC requires substantially weaker entropy regularization
($\alpha=0.1$ or $\alpha=0.01$). At $\times 10$ and $\times 100$ higher
reward magnitudes, VSMC approaches the optimal win probability, while 
still having a lower draw probability and a higher loss probability
than SAC with corresponding entropy regularization.

These results highlight that, under comparable regularization
strength, policy inference and entropy-regularized optimization
induce different trade-offs between exploration and outcome
variance, even in a domain with a known optimal policy.

\subsection{Triangle Tireworld}

In Triangle Tireworld~\citep{LT07}, the agent travels through a
triangular structure of locations connected by directed roads.
The agent must arrive from the initial location to the goal.
With a fixed probability, the agent gets a flat. Some locations
carry a spare tire.  If the agent loaded a spare tire prior to
getting a flat, it changes the tire and continues the travel.
Otherwise, the agent is stuck. The agent gets the reward of 10
for reaching the goal, -10 for getting stuck, and -0.1 for every
step. Triangle Tireworld introduces irreversible stochastic events and
explicit risk–reward trade-offs. Successful policies must plan for
low-probability failures whose consequences cannot be undone.
The domain and the instances are based on the version from the
2014 International planning competition~\cite{VCM+15}.

\begin{table}
\centering
\caption{Triangle Tireworld: policy statistics}
\label{tab:tw}
\begin{tabular}{r l | c c c c }
\textbf{\#} & \textbf{policy} & \textbf{expected return} & \textbf{success probability} \\ \hline
     & VSMC               & $-0.80 \pm 0.52$  & $0.47 \pm 0.03$ \\
1    & VSMC ($0.2\cdot r$)&  $\phantom{-}1.07 \pm 0.32$  & $0.57 \pm 0.02$ \\
     & SAC                &  $\phantom{-}1.27 \pm 0.11$  & $0.58 \pm 0.01$ \\
     & SAC ($\alpha=0.2$) & $-0.26 \pm 0.62$  & $0.50 \pm 0.03$ \\
\hline
     & VSMC               & $-4.51 \pm 0.32$  & $0.29 \pm 0.02$ \\
4    & VSMC ($0.2\cdot r$)& $-1.64 \pm 0.10$  & $0.44 \pm 0.01$ \\
     & SAC                & $-1.90 \pm 0.17$  & $0.43 \pm 0.01$ \\
     & SAC ($\alpha=0.2$) & $-7.25 \pm 1.08$  & $0.15 \pm 0.06$ \\
\hline
     & VSMC               & $-2.91 \pm 0.26$  & $0.39 \pm 0.02$ \\
7    & VSMC ($0.2\cdot r$)& $-1.04 \pm 0.20$  & $0.49 \pm 0.01$ \\
     & SAC                & $-1.35 \pm 0.09$  & $0.48 \pm 0.01$ \\
     & SAC ($\alpha=0.2$) & $-7.49 \pm 0.38$  & $0.28 \pm 0.03$ \\
\hline
     & VSMC               & $-6.14 \pm 0.19$  & $0.22 \pm 0.02$ \\
10   & VSMC ($0.2\cdot r$)& $-5.02 \pm 0.06$  & $0.30 \pm 0.01$ \\
     & SAC                & $-5.27 \pm 0.04$  & $0.28 \pm 0.01$ \\
     & SAC ($\alpha=0.2$) & $-8.83 \pm 0.29$  & $0.14 \pm 0.03$
\end{tabular}
\end{table}
\begin{figure}
\centering
\begin{subfigure}[t]{0.24\textwidth}
  \centering
  \includegraphics[width=0.8\linewidth]{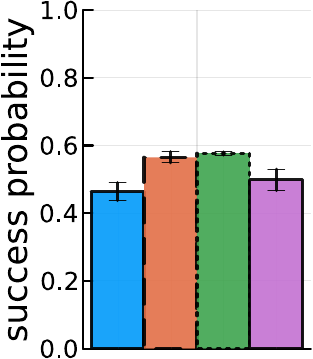}
  \caption{instance 1}
  \label{fig:tw-1}
\end{subfigure}
\begin{subfigure}[t]{0.24\textwidth}
  \centering
  \includegraphics[width=0.8\linewidth]{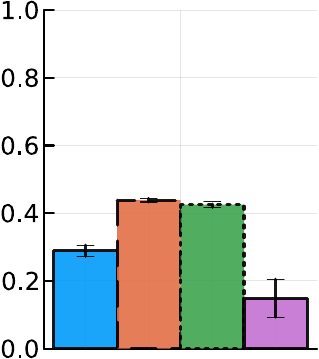}
  \caption{instance 4}
  \label{fig:tw-4}
\end{subfigure}
\begin{subfigure}[t]{0.24\textwidth}
  \centering
  \includegraphics[width=0.8\linewidth]{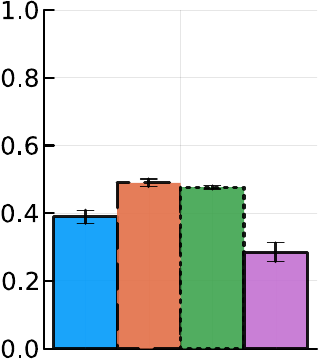}
  \caption{instance 7}
  \label{fig:tw-7}
\end{subfigure}
\begin{subfigure}[t]{0.24\textwidth}
  \centering
  \includegraphics[width=0.8\linewidth]{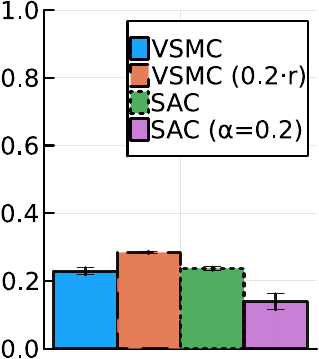}
  \caption{instance 10}
  \label{fig:tw-10}
\end{subfigure}
\caption{Triangle Tireworld: success probabilities}
\label{fig:tw}
\end{figure}
With the original rewards, Triangle Tireworld induces a large return gap
between ``fast but risky'' and ``safe but slow'' behaviors. Under our Bayesian
formulation this makes the posterior highly peaked, yielding low mean and high
variance of expected returns over training runs. Scaling rewards down by a
factor of five reduces this separation, producing a less concentrated
posterior; under this setting VSMC exhibits performance comparable to SAC. Note
that SAC is also sensitive to regularization strength: when SAC's entropy
regularization is equivalently downweighted by setting $\alpha$ to $0.2$, SAC
struggles to find a high-return policy, and the obtained stochastic policy
exhibits similarly low return and high variance.  However, the difference is
semantic rather than just technical: in SAC, the regularization weight is an
algorithm parameter which can be tuned and scheduled to optimize the expected
return on a particular domain; in planning as policy inference, the reward
scale is a part of the domain specification.  Table~\ref{tab:tw} and
Figure~\ref{fig:tw} report results for instances 1, 4, 7, and 10, which are
representative instances spanning the range of difficulty in this domain; the
omitted instances follow the same qualitative trend.

Triangle Tireworld highlights a semantic constraint imposed by Bayesian policy
inference: unlike classical MDP planning, which is invariant to affine reward
scaling, the posterior depends on return magnitudes, so the method works best
when reward scale meaningfully encodes the strength of preferences/regrets
rather than merely ranking policies.

\subsection{Academic Advising}

Academic Advising models a student choosing which courses to take over a
sequence of semesters in order to complete a curriculum. Academic
Advising represents a large combinatorial planning problem with
long horizons and delayed rewards. The branching action space
and stochastic course outcomes create highly multimodal
trajectory returns, providing a test of scalability and behavior
under complex long-term dependencies.  At each semester the
action is to enroll in up to a maximum course load of currently
eligible courses. Course outcomes are stochastic: each enrolled
course is passed with a given probability (otherwise it remains
incomplete and can be retaken later), and passed courses persist
until curriculum completion, at which point the process enters an
absorbing goal state. The reward is specified as
step costs: taking a course incurs -1, retaking a previously
attempted course incurs -2, and if the program is not yet
complete the agent also incurs an additional -5 at every step,
encouraging completion before rollout truncation. In the IPC
2014-derived instances we use, the curriculum contains 10--30
courses; easier (odd-numbered) instances restrict the course
load to at most 1 course per step, while harder (even-numbered)
instances allow up to 2 courses per step. 

\begin{table}
\centering
\caption{Academic Advising: policy statistics}
\label{tab:aa}
\begin{tabular}{r l | c c c c c }
\textbf{\#} & \textbf{policy} & \textbf{expected} &  \multicolumn{2}{c}{\textbf{0.05}}     &  \multicolumn{2}{c}{\textbf{0.95}}    \\
           &                 &   \textbf{return}    &\textbf{quantile} & \textbf{tail mean} & \textbf{quantile} & \textbf{tail mean} \\ \hline
  \raisebox{-0.5\normalbaselineskip}[0pt][0pt]{ 1} & VSMC & $-65.3 \pm 1.3$ & $-141.1 \pm 3.6$ & $-175.7 \pm 5.0$ & $-20.4 \pm 0.5$ & $-18.4 \pm 0.4$\\
    & SAC & $-48.6 \pm 0.8$ & $-84.0 \pm 2.8$ & $-97.1 \pm 2.6$ & $-24.8 \pm 1.2$ & $-21.9 \pm 0.9$ \\
\hline
  \raisebox{-0.5\normalbaselineskip}[0pt][0pt]{ 2} & VSMC & $-98.5 \pm 1.8$ & $-184.2 \pm 5.2$ & $-222.1 \pm 7.1$ & $-45.0 \pm 1.2$ & $-39.5 \pm 1.0$\\
    & SAC & $-106.7 \pm 3.6$ & $-184.6 \pm 8.2$ & $-216.1 \pm 11.1$ & $-52.8 \pm 1.4$ & $-45.7 \pm 1.2$ \\
\hline
  \raisebox{-0.5\normalbaselineskip}[0pt][0pt]{ 3} & VSMC & $-86.7 \pm 1.1$ & $-174.1 \pm 3.0$ & $-207.5 \pm 3.4$ & $-32.6 \pm 0.2$ & $-28.1 \pm 0.5$\\
    & SAC & $-84.4 \pm 2.2$ & $-147.7 \pm 4.5$ & $-174.6 \pm 5.3$ & $-39.0 \pm 1.6$ & $-34.2 \pm 0.7$ \\
\end{tabular}
\end{table}

\begin{figure}
\centering
\begin{subfigure}[t]{0.32\textwidth}
  \centering
  \includegraphics[width=0.8\linewidth]{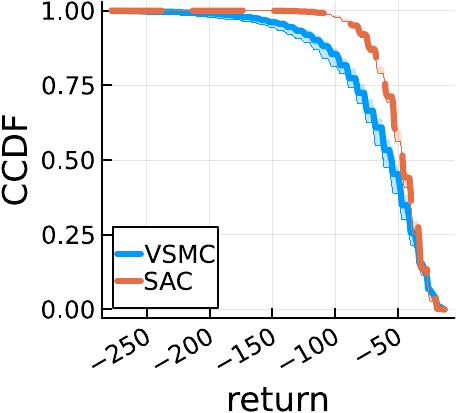}
  \caption{instance 1}
  \label{fig:aa-1}
\end{subfigure}
\begin{subfigure}[t]{0.32\textwidth}
  \centering
  \includegraphics[width=0.8\linewidth]{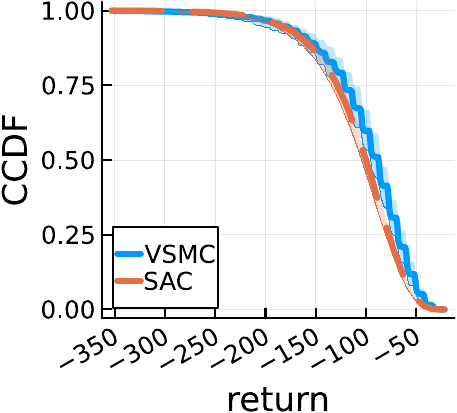}
  \caption{instance 2}
  \label{fig:aa-2}
\end{subfigure}
\begin{subfigure}[t]{0.32\textwidth}
  \centering
  \includegraphics[width=0.8\linewidth]{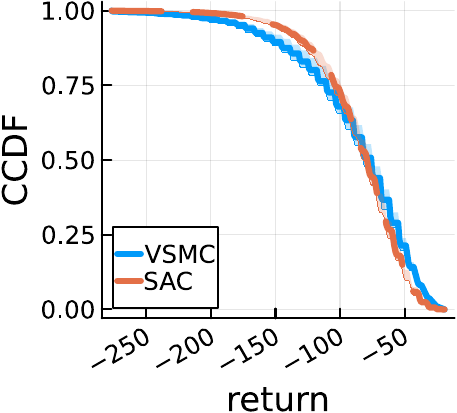}
  \caption{instance 3}
  \label{fig:aa-3}
\end{subfigure}
\caption{Academic Advising: trajectory return distributions}
\label{fig:aa}
\end{figure}
Without a non-trivial baseline policy or a domain-specific heuristic, SAC and
VSMC reliably find policies with a non-negligible probability of completing the
program for instances 1--3. For harder instances, either the variation between
runs is very high, or the policy / policy distribution concentrates around a
random walk that minimizes the per-step cost but does not lead to program
completion. Table~\ref{tab:aa} summarizes policy statistics and
Figure~\ref{fig:aa} shows trajectory return distributions for instances 1--3.
The comparison is most informative at the level of return distributions rather
than mean return alone: VSMC and SAC often achieve comparable average
performance, but VSMC induces heavier-tailed trajectory return distributions,
as manifested by the 0.05 and 0.95 quantiles and their conditional tail means.
This is consistent with the broader claim of the paper that the main difference
lies in how the two objectives represent and act under trajectory uncertainty,
not in uniform reward dominance.

The Academic Advising results demonstrate that the differences
between the two approaches persist in larger combinatorial
settings, where long horizons and stochastic outcomes amplify
differences in how trajectory uncertainty is represented.

\section{Discussion}

We view episodic MDP planning as posterior inference over policies,
with expected return defining an unnormalized posterior density
(Eq.~\eqref{eqn:logp_pi}). The resulting posterior concentrates on
return-maximizing policies, while its dispersion captures uncertainty
about optimal behavior.

\paragraph{Sources of uncertainty.}
With a posterior over deterministic policies, three sources of
uncertainty are disentangled:
\begin{enumerate}[label=(\roman*)]
\item \textbf{aleatoric} transition randomness, sampled forward and
appearing as noise in the Monte Carlo estimate of policy log-probability
(Eq.~\eqref{eqn:hat_logp_pi});
\item \textbf{epistemic} uncertainty over optimal behavior, represented
by posterior dispersion; and
\item \textbf{execution-time stochasticity}, obtained by marginalizing
over deterministic policies. In exact form, posterior-predictive control
is therefore a structured form of Thompson sampling: actions randomize
only to the extent that multiple deterministic behaviors remain
plausible.
\end{enumerate}

\paragraph{Inference mechanics.}
Once planning is cast as inference with an intractable target density,
the sequential structure of returns makes SMC a natural fit, and VSMC
provides a principled variational objective. Policy inference, however,
requires two adaptations: enforcing policy consistency under revisitation
(by sampling each state's action only on first visit) and coupling
transition randomness across particles within a sweep so that weights
reflect policy differences rather than independent simulator noise.

\paragraph{Relation to control-as-inference and entropy-regularized RL.}
Compared to trajectory-centric control-as-inference formulations, our
latent variable is the policy itself. Compared to
entropy-regularized RL, stochasticity reflects posterior uncertainty
over deterministic behaviors rather than an explicit entropy preference.

\paragraph{Empirical takeaways.}
Across domains, the comparison to SAC should be read primarily as a
comparison of objectives and the behaviors they induce, rather than as a
claim that policy VSMC is a uniformly stronger optimizer of expected
reward. The worked examples show that policy-posterior inference and
entropy-regularized optimization can coincide in simple one-step settings,
but need not coincide when stochastic continuation or state revisitation
makes policy-level commitment differ from action-level resampling. The
empirical results are consistent with this distinction, while also
reflecting the effects of approximation, optimization, and reward scale. In
grid worlds, the variationally approximated induced policy avoids
boundary-directed actions that can increase entropy under SAC while
reducing goal reachability. This illustrates the basic mechanism behind
the empirical differences: entropy-regularized RL randomizes at the
action level, whereas policy inference places posterior mass on complete
deterministic policies whose repeated statewise choices must themselves
be plausible under the return model. In Blackjack, matching VSMC-like behavior
requires substantially weaker entropy regularization in SAC, suggesting
that the default entropy-regularized objective occupies a different
operating point. In Triangle Tireworld, the sensitivity of VSMC to reward
scale reveals a substantive property of the Bayesian formulation: return
magnitudes control posterior concentration. In Academic Advising, the main
contrast is distributional: both methods can achieve similar mean
performance on some instances, but they differ in tail behavior and
variability. Taken together, these results support the narrower conclusion
that policy inference and entropy-regularized optimization induce
qualitatively different behaviors even when aggregate returns are similar,
rather than a uniform dominance claim for either method.

\paragraph{Scalability considerations.}
The method inherits the usual scaling of particle inference: runtime grows with
the rollout horizon and the number of particles. The additional bookkeeping for
policy and transition consistency adds memory proportional to the number of
distinct state--action--visit-count queries encountered during a sweep. This
overhead is modest in the discrete domains considered here, where revisitation
occurs, but may become relevant in very large state spaces or domains with
little revisitation. This cost is distinct from resampling itself: prior VSMC
results~\citep{NLR+18} show that resampling can make better use of small
particle budgets than IWAE-style variational importance sampling. Overall,
scalability is governed primarily by the standard tradeoffs among particle
count, rollout length, and effective state-space size.

\paragraph{Scope and extensions.}
We focus on discrete state spaces to make revisit bookkeeping and shared
transition caching explicit. Extensions beyond the finite/discrete setting
require specifying a proper reference measure on the chosen measurable
policy class; bounded finite-horizon rewards then guarantee a finite
normalizing constant. For infinite action spaces, this reference measure
should be regarded as part of the action-space geometry needed to interpret
rewards as log-preference densities over policies. Subject to this
qualification, determinism can be enforced via the policy representation or
a hashable state abstraction, and shared stochasticity can be implemented via common random numbers or
keyed random streams. Finally, strict memoization can be relaxed via
\emph{stochastic memoization}, allowing occasional resampling on revisits
to reduce the brittleness of committing to an early no-op action.

\bibliographystyle{plainnat}
\bibliography{refs}
\newpage

\appendix

\section*{Algorithmic Details}

\begin{algorithm}[H]
\caption{VSMC sweep for deterministic-policy inference. The procedure returns the surrogate
objective in Eq.~\eqref{eqn:det-surrogate-objective-stratified}. Resampling is shown in the
canonical SMC form; the implementation resamples adaptively (ESS threshold of $0.5$).}
\label{alg:detvsmc}
\begin{algorithmic}[1]
\Procedure{\textsc{PolicyVSMC}}{$s_1$, \textsc{Step}, \textsc{Reward}, $H$, $N$, $q_\theta$}
    \State $M_T \gets \varnothing$
    \For{$i\gets 1$ to $N$}
        \State $s^{(i)} \gets s_1$
        \State $M_A^{(i)} \gets \varnothing$ \Comment action memoization: $s\mapsto a$
        \State $M_C^{(i)} \gets \varnothing$ \Comment counts: $(s,a)\mapsto k$
    \EndFor
    \State $\ell_{1:H} \gets 0$ \Comment per-step log mean weight increments
    \State $g_{1:H} \gets 0$ \Comment per-step proposal log-prob terms (first-visit only)

    \For{$t\gets 1$ to $H$}
        \For{$i\gets 1$ to $N$}
            \State $s \gets s^{(i)}$
            \If{$s \in \dom M_A^{(i)}$} \Comment revisit: reuse memoized action
                \State $a \gets M_A^{(i)}(s)$
                \State $\log p_a \gets 0,\;\log q_a \gets 0$
            \Else \Comment first visit: sample and memoize
                \State $a \sim q_\theta(\cdot\mid s)$,\; $M_A^{(i)}(s) \gets a$
                \State $\log p_a \gets -\log|A|$,\; $\log q_a \gets \log q_\theta(a\mid s)$
                \State $g_t \gets g_t + \log q_a$
            \EndIf

            \State $k \gets M_C^{(i)}(s,a) + 1$ \Comment default $M_C^{(i)}(s,a)=0$ if absent
            \State $M_C^{(i)}(s,a) \gets k$

            \If{$(s,a,k)\in \dom M_T$} \Comment coupled transition randomness
                \State $s' \gets M_T(s,a,k)$
            \Else
                \State $s' \gets \textsc{Step}(s,a)$
                \State $M_T(s,a,k) \gets s'$
            \EndIf

            \State $w^{(i)} \gets \textsc{Reward}(s,a,s') + \log p_a - \log q_a$
            \State $s^{(i)} \gets s'$
        \EndFor

        \State $\ell_t \gets \log \sum_{i=1}^N \exp(w^{(i)}) - \log N$
        \State $(s^{(1:N)}, M_A^{(1:N)}, M_C^{(1:N)}) \gets \textsc{Resample}\big((s^{(1:N)}, M_A^{(1:N)}, M_C^{(1:N)}),\, w^{(1:N)}\big)$
    \EndFor

    \State $\log \hat Z \gets \sum_{t=1}^H \ell_t$
    \State $\log \hat Z_{H+1} \gets 0$
    \For{$t\gets H$ down to $1$}
        \State $\log \hat Z_t \gets \ell_t + \log \hat Z_{t+1}$ \Comment suffix sums for Eq.~\eqref{eqn:det-surrogate-objective-stratified}
    \EndFor

    \State \Return $\log \hat Z + \sum_{t=1}^H \left(\overline{\log \hat Z_t}\cdot g_t\right)$
\EndProcedure
\end{algorithmic}
\end{algorithm}

\end{document}